%% file: main.tex
\RequirePackage[svgnames]{xcolor}
\PassOptionsToPackage{inline}{enumitem}

\documentclass[11pt,letterpaper]{mystyle}

\usepackage[all]{hypcap}
\usepackage[svgnames]{xcolor}
\usepackage[comma,authoryear,compress]{natbib}
\usepackage{hyperref}[citecolor=lightblue]

\hypersetup{
    colorlinks = true,
    citecolor = {YaleBlue},
}

\usepackage[utf8]{inputenc} 
\usepackage[T1]{fontenc}    
\usepackage{booktabs}       
\usepackage{amsfonts}       
\usepackage{nicefrac}       
\usepackage{microtype}      
\usepackage{xcolor}         
\definecolor{VeriGateSOne}{HTML}{C99D15}
\definecolor{VeriGateSTwo}{HTML}{3A71D0}
\definecolor{VeriGateSThree}{HTML}{4A8330}
\DeclareRobustCommand{\verigatebadge}[2]{%
  \kern0.03em\tikz[baseline=(badge.base)]\node[%
    circle,%
    draw=none,%
    fill=#1,%
    text=black,%
    font=\bfseries\tiny,%
    minimum size=1.42em,%
    inner sep=0pt%
  ] (badge) {#2};\kern0.02em%
}
\usepackage{tcolorbox}

\usepackage{xspace}
\newcommand{\ourstitle}{VeriGate\xspace}
\newcommand{\ours}{\textup{\textsc{\ourstitle}}\xspace}
\newcommand{\oursfull}{\textup{Verifier-Gated Step-Level GRPO}\xspace}

\usepackage{amsmath}
\usepackage{amssymb}
\usepackage{mathtools}
\usepackage{amsthm}
\usepackage{makecell}
\usepackage{enumitem}
\usepackage{tikz}
\usepackage{booktabs}
\usepackage{wrapfig}
\usetikzlibrary{arrows.meta, positioning, shapes.geometric, fit}
\usepackage{circledsteps}
\usepackage{subcaption}
\usepackage{algorithm}
\usepackage{algorithmicx}
\usepackage{algpseudocode}
\usepackage{microtype}
\usepackage{graphicx}
\expandafter\def\csname ver@subfig.sty\endcsname{}
\usepackage{booktabs} %
\usepackage{float}
\usepackage{bigstrut}

\usepackage{mathrsfs}
\usepackage{nicefrac}
\usepackage{dsfont}
\usepackage{enumitem}
\usepackage{subcaption}
\usepackage{graphicx,subfig}
\usepackage{cleveref}
\usepackage{bxcoloremoji}

\usepackage{float}

\usepackage[utf8]{inputenc} %
\usepackage[T1]{fontenc}    %
\usepackage{hyperref}       %
\usepackage{url}            %
\usepackage{booktabs}       %
\usepackage{amsfonts}       %
\usepackage{nicefrac}       %
\usepackage{microtype}      %
\usepackage{graphicx}
\usepackage{subcaption} 
\usepackage{amssymb}
\usepackage{fdsymbol}
\usepackage{wrapfig}
\usepackage{lipsum}
\usepackage{enumitem}
\usepackage{stackengine}
\usepackage[font=small,labelfont=bf]{caption}
\usepackage{color}
\usepackage{adjustbox}

\usepackage{rotating}
\usepackage{makecell}

\input{macro}
\newtcolorbox{AIbox}[2][]{aibox,title=#2,#1}
\definecolor{lightblue}{rgb}{0.22,0.45,0.70}%
\definecolor{Gray}{gray}{0.95}
\definecolor{Cornsilk}{rgb}{1.0, 0.97, 0.86}

\usepackage{amsmath}

\usepackage[all]{hypcap}

\title{\ourstitle: Verifier-Gated Step-Level Supervision for GRPO}

\runningtitle{\ourstitle: Verifier-Gated Step-Level Supervision for GRPO}

\author{
  Aakriti Agrawal$^{*1}$,
  Minghui Liu$^{*1}$, and
  Furong Huang
}

\affil[1]{University of Maryland, College Park}
\affil[*]{equal contribution}

\correspondingauthor{Aakriti Agrawal; Email \href{mailto:agrawal5@umd.edu}{agrawal5@umd.edu}}

\begin{document}

\begin{abstract}
Group Relative Policy Optimization (GRPO) is an effective recipe for training reasoning models with verifier-based outcome rewards, but its supervision is sparse: when all sampled trajectories for a prompt receive the same verifier reward, the group-relative advantage collapses to zero and learning stalls. Outcome-only rewards also provide no step-level credit assignment, limiting exploration and making it harder to learn robust reasoning. We present \ours (\oursfull), a verifier-gated extension of GRPO that addresses these limitations with three design choices. First, \ours keeps the verifier in charge whenever verifier rewards induce a meaningful preference among sampled trajectories, and uses process supervision only when verifier rewards are degenerate. Second, instead of collapsing Process Reward Model (PRM) step scores into a single trajectory reward, \ours converts them into future-cumulated rewards to assign continuation-aware credit. Third, \ours transforms these rewards into group-normalized token-level advantages, restoring informative gradients and fine-grained credit assignment while remaining less susceptible to reward hacking than methods that optimize aggregated PRM scores.
Empirically, training on MATH with 1.5B and 7B Qwen2.5-Instruct models and evaluating on six reasoning benchmarks, \ours improves average accuracy by about $20\%$ and $12\%$ for 1.5B and 7B models respectively, substantially reduces zero-gradient failures, decreases reward-hacking behavior, and improves reasoning quality relative to outcome-only GRPO and PRM-as-outcome baselines.

\coloremojicode{1F3E0} \textbf{Projects}: \href{https://tinyurl.com/Aakriti-agrawal-verigate}{https://tinyurl.com/Aakriti-agrawal-verigate}

\github{} \textbf{Code Repository}: \href{https://github.com/umd-huang-lab/VeriGate}{https://github.com/umd-huang-lab/VeriGate}

\coloremojicode{1F4E7} \textbf{Contact}: \href{mailto:agrawal5@umd.edu}{agrawal5@umd.edu}
\end{abstract}

\maketitle
\vspace{3mm}
\input{sections/Introduction}
\input{sections/Problem}

\input{sections/Method}

\input{sections/Experiments}

\input{sections/Related_Works}
\input{sections/Conclusion}

\section{Acknowledgement}
Agrawal, Liu and Huang are supported by DARPA HR001124S0029-AIQ-FP-019,  National Science Foundation TRAILS Institute (2229885). Private support was provided by Open Philanthropy and Apple. The Authors acknowledge the National Artificial Intelligence Research Resource (NAIRR) Pilot and [insert the resources supporting your project here] for contributing to this research result.


\newpage
\bibliography{references}

\appendix
\input{sections/Appendix}
\appendix
\end{document}

%% file: macro.tex
\definecolor{blanchedalmond}{rgb}{1.0, 0.92, 0.8}
\definecolor{carmine}{rgb}{0.59, 0.0, 0.09}
\definecolor{lightblue}{rgb}{0.22,0.45,0.70}%

\newtheorem{theorem}{Theorem}[section]

\newtheorem{proposition}[theorem]{Proposition}

\renewcommand{\mathbf}{\boldsymbol}

\makeatletter
\def\Ddots{\mathinner{\mkern1mu\raise\p@
\vbox{\kern7\p@\hbox{.}}\mkern2mu
\raise4\p@\hbox{.}\mkern2mu\raise7\p@\hbox{.}\mkern1mu}}
\makeatother

\definecolor{amaranth}{rgb}{0.9, 0.17, 0.31}
\definecolor{antiquebrass}{rgb}{0.8, 0.58, 0.46}
\definecolor{antiquefuchsia}{rgb}{0.57, 0.36, 0.51}
\definecolor{chromeyellow}{rgb}{0.31, 0.47, 0.26}

\newcommand{\github}{\raisebox{-1.5pt}{\includegraphics[height=1.05em]{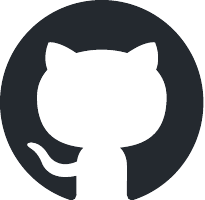}}}

%% file: sections/Introduction.tex
\vspace{-4mm}

\begin{figure}[!htbp]
    \centering
    \includegraphics[width=0.86\linewidth, trim={0 0 0 0}, clip]{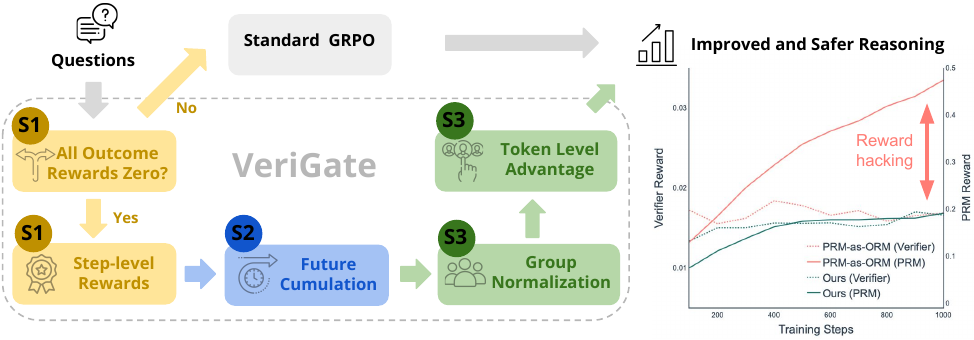}
    \vspace{-1em}
    \caption{\textbf{Overview of \ours.} \textbf{Left:} \ours integrates process supervision into GRPO through three design choices: \verigatebadge{VeriGateSOne}{S1} gate PRM supervision by verifier informativeness, using standard GRPO when verifier rewards induce a preference and activating the PRM only for all-zero verifier groups; \verigatebadge{VeriGateSTwo}{S2} use future-cumulated rewards for continuation-aware credit assignment; and \verigatebadge{VeriGateSThree}{S3} convert PRM feedback into group-normalized token-level advantages, restoring informative learning signals without overriding reliable verifier feedback. \textbf{Right:} Direct optimization of aggregated PRM rewards leads to reward hacking, with PRM scores rising while verifier scores remain flat. \ours mitigates this behavior while preserving step-level supervision.}
    \vspace{-1em}
    \label{fig:intro-fig}
\end{figure}

\vspace{-4mm}
\section{Introduction}
\vspace{-2mm}
Reinforcement learning with verifiable rewards (RLVR) has become a core driver of progress in large reasoning models (LRMs) \citep{jaech2024openai, guo2025deepseek, team2025kimi}. In domains such as mathematics, GRPO \citep{shao2024deepseekmath} is especially appealing because it trains on verifier-based outcome rewards without a separate reward model or critic \citep{lambert2024tulu, guo2025deepseek}. This makes it efficient, relatively stable, and grounded in verifiable correctness. But the same outcome-only design is also GRPO's main weakness on hard reasoning problems. Because GRPO assigns a single scalar reward to the full reasoning trajectory, it suffers from zero-gradient collapse when sampled trajectories receive identical verifier rewards \citep{sun2025rl}, provides no step-level credit assignment so all tokens in a correct solution are reinforced equally \citep{zhang2025interplay}, and offers weak guidance for exploration and generalization \citep{yue2025does, bragg2018can, chen2025exploration, ruan2025unveiling, wu2026reasoning, yan2026spurious}. In short, verifier rewards are trustworthy but too sparse to teach the reasoning process, which motivates adding denser process supervision.

Process Reward Models (PRMs) are a natural way to provide this denser supervision: they score intermediate reasoning steps and can provide feedback that outcome-only training lacks \citep{lightman2023let, wang2024math, luo2024improve, zhang2024rest, zhang2025interplay, yang2025treerpo}. In principle, this should help learning on hard prompts, improve credit assignment, and encourage more faithful reasoning trajectories. \textbf{But integrating PRMs into GRPO is not straightforward.} It requires addressing three coupled \textbf{challenges}. \textbf{(C1) Preserve verifier authority.} PRMs are learned models trained on imperfect supervision, so they inherit biases and blind spots, and the rewards they generate can be noisy; using them naively can let an imperfect signal override a trustworthy verifier and invite reward hacking. \textbf{(C2) Avoid brittle reward aggregation.} PRMs score individual steps, whereas GRPO is naturally driven by trajectory-level rewards; collapsing step scores into a single scalar induces compounding errors, because one biased step can distort the effective reward assigned to the whole response while discarding fine-grained credit assignment \citep{zhang2025interplay}. \textbf{(C3) Convert dense feedback into a GRPO-compatible objective.} Even if we keep PRM feedback dense, it still must be converted into a GRPO-compatible learning signal. Existing approaches typically fail on one or more of these fronts: PRM-as-outcome \citep{zhang2025interplay} methods replace the verifier with an aggregated PRM score, while intrinsic process supervision based methods use token-level, entropy which is unreliable and heavily constrained by the policy model's own capabilities \citep{anonymous2025cut, baker2025monitoring, denison2024sycophancy}.


This raises a central question:
\textit{Can we use dense step-level process feedback to fix GRPO's zero-gradient and credit-assignment failures, without giving up the reliability of verifier-based rewards?}

\textbf{Verifier-gated process supervision.}
We answer this question with \ours (\oursfull), a verifier-gated extension of GRPO that uses process supervision only when the verifier is uninformative. Concretely, \ours resolves the three challenges with three corresponding design choices. \textbf{(S1) Gate PRM supervision by verifier informativeness.} When verifier rewards induce a meaningful preference among sampled trajectories, \ours reduces exactly to standard GRPO; PRM supervision is used only when all sampled trajectories receive zero verifier reward and GRPO would otherwise produce zero reward-driven gradients. \textbf{(S2) Use future-cumulated rewards for continuation credit.} Instead of collapsing PRM step scores into a single trajectory-level scalar, \ours converts them into future-cumulated rewards so that each token is credited according to the quality of the continuation it enables. \textbf{(S3) Convert PRM feedback into group-normalized token-level advantages.} \ours then transforms these future-cumulated rewards into a dense learning signal that directly replaces the degenerate trajectory-level GRPO update.

This design preserves the strongest property of RLVR while addressing its weakest point. Whenever correctness is distinguishable, \ours optimizes only the verifier signal. If the verifier assigns zero reward to all sampled responses and thus provides no reward-driven learning signal, \ours falls back to PRM-based process supervision. Because the PRM enters only through centered, normalized advantages built from future-cumulated rewards, \ours is invariant to prompt-wise positive affine transformations of those rewards, ruling out a common class of PRM reward hacking. The result is a simple drop-in modification to GRPO that improves exploration, yields more faithful reasoning updates, and substantially reduces the reward-hacking surface relative to PRM-as-outcome baselines.

\textbf{Our Contributions.}
\begin{itemize}[leftmargin=*,topsep=0in]
\setlength\itemsep{1pt}
\setlength\parsep{0pt}
\setlength\parskip{0pt}

\item \textbf{Verifier-gated process supervision for GRPO.}
We introduce \ours, a simple extension of GRPO that preserves standard verifier-based updates when they are informative and falls back to PRM-based token-level supervision only for all-zero verifier groups.

\item \textbf{Continuation-aware token-level credit assignment.}
Instead of collapsing PRM step scores into a single trajectory reward, \ours converts them into future-cumulated rewards and group-normalized token-level advantages, restoring dense learning signals precisely where outcome-only GRPO is silent.

\item \textbf{Reduced susceptibility to PRM reward hacking.}
We show that \ours is invariant to prompt-wise positive affine transformations of future-cumulated PRM rewards, limiting over-optimization of imperfect PRM signals relative to objectives that optimize aggregated PRM rewards.

\item \textbf{Strong empirical gains.}
Training on MATH and evaluating on six reasoning benchmarks, \ours improves average accuracy by $20\%$ and $12\%$ for 1.5B and 7B model respectively, substantially reduces zero-gradient failures and reward-hacking behavior, and improves reasoning quality relative to outcome-only GRPO and PRM-as-outcome baselines.
\end{itemize}

%% file: sections/Problem.tex
\vspace{-2mm}
\section{Why Outcome-Only GRPO and Naive Process Supervision Fail}
\label{sec:problems}
\vspace{-2mm}
Our method starts from two observations. Outcome-only GRPO breaks down on the hard prompts where extra guidance is most needed. The most natural fix---replacing verifier rewards with PRM-derived trajectory rewards---introduces a different failure mode by over-optimizing an imperfect learned signal. For background on GRPO and the distinction between outcome and process rewards, see Appendix Section~\ref{sec:background}. We make three points.

\begin{wraptable}{r}{0.4\textwidth}
    \vspace{-6pt}
    \centering
    \caption{Percentage of zero-reward prompts before and after one epoch of GRPO training for the Qwen2.5-Instruct models.}
    \label{tab:model_training}
    \resizebox{\linewidth}{!}{
    \begin{tabular}{l l c c}
      \toprule
      \textbf{Model} & \textbf{Dataset} & \textbf{Base} & \textbf{1 epoch} \\
      \midrule
      1.5B & MATH & 45.6 & 39.9 \\
      1.5B & DAPO-MATH & 82.9 & 78.2 \\
      \midrule
      7B & MATH & 41.7 & 39.3 \\
      7B & DAPO-MATH & 79.2 & 77.5 \\
      \bottomrule
    \end{tabular}
    }
    \vspace{-6pt}
\end{wraptable}

\noindent\textbf{Outcome-only GRPO suffers from reward degeneracy.}
GRPO fails when all sampled trajectories in a group receive the same reward. If $r_1 = \dots = r_G$ for a group of size $G$, there is no within-group preference signal, so $A_i = 0$ for all $i$ and the \emph{advantage-driven} part of the GRPO update collapses to zero. This is the core \emph{zero-gradient problem} discussed in prior work~\cite{sun2025rl}. Strictly speaking, the full gradient of the regularized objective need not vanish because the KL term still pushes the policy toward the reference model. Thus, on such degenerate batches, learning from reward disappears while only the KL-driven update remains, which can lead to policy drift or even unlearning on hard prompts.

For binary rewards $r_i \in \{0,1\}$, the probability that the reward-driven GRPO update is degenerate is
    $\mathbb{P}\!\left(A_1=\cdots=A_G=0\right) = p^G + (1 - p)^G,$
where $p = \mathbb{P}(r = 1)$ and $G$ is the group size. This probability is high when $p$ is close to $0$ or $1$. The $p \approx 1$ regime corresponds to easy prompts and is mostly harmless. The real problem is $p \approx 0$: on hard prompts, almost all sampled trajectories receive zero reward, updates vanish, and learning stalls exactly where the model most needs signal. Table~\ref{tab:model_training} shows that many prompts remain in this degenerate regime even after one epoch of GRPO training.




\noindent\textbf{Outcome rewards give no step-level credit assignment.}
Outcome-only GRPO applies the same trajectory-level advantage to every token in a sampled response. If a trajectory $y=(y_1,\dots,y_T)$ receives advantage $A_i$, the policy gradient scales every token log-probability by that same scalar. It therefore cannot separate genuinely useful steps from incidental or spurious ones, which is especially problematic in long mathematical reasoning where only a few intermediate steps determine the final answer. Verifier rewards are thus reliable but poor at local credit assignment: they say whether the answer is correct, not which steps deserve credit or blame. As a result, training can reinforce shortcuts that preserve accuracy without improving the reasoning process.

\noindent\textbf{Naive PRM integration fixes sparsity but creates new risks.}\label{sec:prm_problems}
The most direct response to reward sparsity is to aggregate step-level PRM scores into a single trajectory reward and optimize it instead of the verifier reward. Following~\cite{zhang2025interplay}, we call this family of approaches \textsc{PRM-as-ORM}, because it treats a process reward model as if it were an outcome reward model. Formally, let a PRM produce reasoning-step scores $r_\phi(x,y_{<j},y_j)$. A trajectory reward $R(x,y)$ is then often formed either by the \textit{minimum}, $R(x,y)=\min_{j=1}^{S} r_\phi(x,y_{<j},y_j)$, or by the \textit{product}, $R(x,y)=\prod_{j=1}^{S} r_\phi(x,y_{<j},y_j)$.

PRM-as-ORM is attractive because it replaces sparse verifier rewards with a dense learned signal. But it is risky for two reasons. First, any systematic PRM bias is lifted directly to the trajectory level. Second, because the objective is an \emph{absolute} PRM-derived scalar, the policy can improve reward by exploiting generic patterns the PRM prefers even when they do not improve correctness. This is exactly the setting where reward hacking becomes likely~\citep{anonymous2025cut}.
Figure~\ref{fig:prm_score_dapo_math} illustrates this risk. Even strong PRMs are much better calibrated on easier benchmarks than on harder ones. On MATH-500, PRM product scores align reasonably well with correctness. On OlympiadBench, and especially on DAPO-MATH-17k, the distributions are far more misaligned with ground truth, including many correct trajectories that still receive low PRM scores. Directly optimizing such a reward can therefore amplify PRM-specific biases rather than improve genuine reasoning quality.

\begin{figure}[t]
    \centering
    \includegraphics[width=0.99\linewidth, trim={0 0 0 0.5cm}]{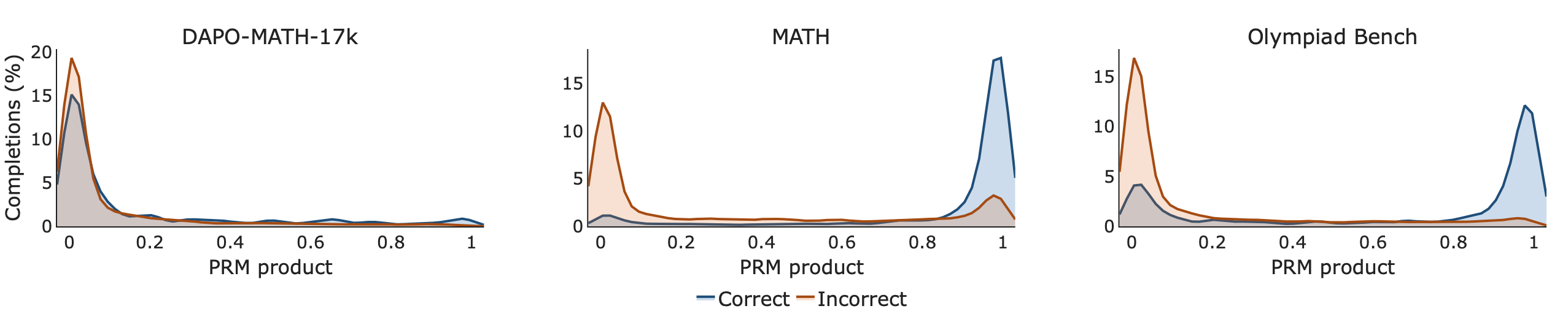}
    \vspace{-0.5em}
    \caption{\textbf{Distribution of PRM product scores by correctness} on MATH-500 \cite{hendrycks2021measuring}, OlympiadBench \cite{he2024olympiadbench}, and DAPO-MATH-17k \cite{yu2025dapo}. Ideally, correct traces should receive PRM scores close to 1, while incorrect traces should receive scores close to 0. On MATH-500, PRM scores are reasonably aligned with ground-truth correctness. On the more challenging datasets, especially DAPO-MATH-17k, the distributions become substantially more skewed, indicating a larger mismatch between PRM preference and actual correctness.}
    \label{fig:prm_score_dapo_math}
    \vspace{-1em}
\end{figure}

These observations motivate our design choice: use PRM supervision only when verifier rewards are uninformative, and keep it as a relative, step-level signal rather than a replacement outcome reward.

%% file: sections/Method.tex
\vspace{-2mm}
\section{Verifier-Gated Token-Level Supervision for GRPO}
\label{sec:method}
\vspace{-2mm}
Following the design choice, we introduce \ours, a verifier-gated extension of GRPO that preserves verifier-based updates whenever they are informative and invokes PRM supervision only when standard GRPO would otherwise receive no learning signal. Implementing this switch requires resolving three design challenges raised in the introduction: how to use PRMs without letting them override a trustworthy verifier, how to prevent brittle trajectory-level PRM aggregation from distorting credit assignment, and how to integrate dense PRM feedback into a GRPO-style update. We address these challenges in turn.

\vspace{-2mm}
\subsection{Challenge 1: Using PRMs Without Letting Them Replace the Verifier}
\vspace{-2mm}

The design principle behind \ours is simple:
\textit{Use verifier rewards whenever they induce a meaningful preference among sampled trajectories; use token-level PRM supervision only when all sampled trajectories receive zero verifier reward and GRPO would otherwise produce zero reward-driven gradients.}

\textbf{Gating Rule.} Concretely, for a prompt $x$, let $\{r_i\}_{i=1}^G$ denote the verifier rewards for the sampled group. If at least one sampled trajectory is correct and at least one is incorrect, then the induced GRPO advantages are informative and \ours reduces exactly to standard GRPO. If instead $r_1 = \cdots = r_G = 0$, then GRPO's group-relative advantages collapse to zero on that hard prompt and the policy receives no reward-driven update. In that case, \ours switches to PRM-based token-level supervision derived from future-cumulated step rewards. We do not activate the PRM when $r_1 = \cdots = r_G = 1$, since those fully correct groups are already solved and are not the failure mode of interest.

This gating rule resolves the first challenge. First, it preserves the strongest property of RLVR: whenever correctness is distinguishable, policy improvement is governed entirely by the verifier. Second, it restricts the use of imperfect PRM signals to exactly those hard prompts where the verifier assigns zero reward to every sampled trajectory. In this sense, PRM supervision acts as a fallback recovery mechanism rather than a replacement source of truth.
\vspace{-2mm}
\subsection{Challenge 2: Avoiding Brittle Trajectory-Level PRM Aggregation}
\vspace{-2mm}

The second challenge is that PRMs score completed reasoning \emph{steps}, while the policy is trained at the token level. A naive solution is to aggregate all PRM step scores into a single trajectory-level scalar and optimize that value directly. But this makes the learning signal brittle: a single biased step score can distort the effective reward assigned to the whole trajectory, and the model loses the main benefit of process supervision, namely fine-grained credit assignment. Our solution is to keep PRM feedback at the step level and translate it into token-level continuation credit, which we call Future-Cumulated Token Rewards (FCTR). 

\textbf{Future-Cumulated Token Rewards (FCTR).}
Let trajectory $y_i$ contain $S_i$ reasoning steps, and let $r_{i,j}$ denote the PRM reward for step $j$. For each step, we define the future-cumulated reward
\begin{equation}
    c_{i,j} = \sum_{k=j}^{S_i} r_{i,k}.
\end{equation}
If token $t$ belongs to step $j$, then token $t$ inherits the value $c_{i,j}$. Intuitively, this means a token is rewarded not only for helping produce a good local step, but also for setting up strong downstream reasoning. Earlier tokens therefore receive credit for enabling later high-quality steps, while tokens in flawed prefixes are penalized through the poor continuation they induce.

This construction resolves the second challenge by avoiding a single absolute PRM score for the entire response. Instead, each step is evaluated through the quality of the continuation it supports. This future cumulation is useful because PRM scores can be noisy and locally myopic: a step that looks good in isolation may later lead nowhere, while an uncertain intermediate step may enable a correct solution. Summing future rewards therefore uses the continuation as extra evidence about whether a step was genuinely useful, which smooths out local PRM noise and gives more credit to steps that support strong downstream reasoning. Our experiments also include a simplified ablation that removes future cumulation, allowing us to isolate the value of this continuation-aware credit assignment.

\vspace{-2mm}
\subsection{Challenge 3: Integrating Dense PRM Feedback into GRPO}
\vspace{-2mm}

The third challenge is that standard GRPO is formulated with a trajectory-level advantage, whereas PRM supervision is useful precisely because it is dense and step-sensitive. Rather than forcing PRM feedback back into a single scalar reward, \ours replaces the degenerate trajectory-level signal with normalized token-level advantages derived from the future-cumulated rewards above.

\textbf{Token-Level Advantage.}
Let $c_{i,j}$ denote the future-cumulated reward assigned to every token in step $j$ of trajectory $i$. We define the corresponding PRM-derived token-level advantage as
\begin{equation}
    A_{i,j} = \frac{c_{i,j} - \bar c}{\sigma(c)},
\end{equation}
where $\bar c$ is a prompt-level baseline and $\sigma(c)$ is the standard deviation computed over all future-cumulated rewards in the sampled group.

Specifically,
\begin{equation}
    \bar c = \frac{\sum_{i=1}^{G}\sum_{j=1}^{S_i} c_{i,j}}{\sum_{i=1}^{G}S_i}.
\end{equation}
Every token in step $j$ inherits the same advantage $A_{i,j}$. Intuitively, $A_{i,j}$ measures whether the continuation enabled by tokens in step $j$ is better or worse than the \emph{average} continuation produced for that prompt.

This construction resolves the third challenge. First, it provides dense credit assignment within a trajectory instead of assigning the same scalar to every token. Second, it compares token-level learning signals \emph{relative} to other sampled continuations for the same prompt, rather than optimizing an absolute PRM score. Third, standardization by $\sigma(c)$ keeps update magnitudes stable without introducing a learned value function.

\textbf{Intuition Behind Why \ours's Token Advantage Is Harder to Reward Hack.} A common form of reward hacking is \emph{prompt-wise reward distortion}: the policy discovers a stylistic or templated behavior that systematically shifts or rescales the reward assigned to nearly every token for a prompt without improving actual correctness. This is distrinctly different than outcome-level PRM aggregation strategy which optimizes an absolute PRM-derived scalar. On the other hand \ours uses centered token-level advantages computed from future-cumulated rewards.

\begin{proposition}[Advantage invariance to prompt-wise positive affine PRM transformations]
\label{prop:invariance_inflation}
Suppose a policy update induces a prompt-wise positive affine transformation
$
    c'_{i,j} = a(x)c_{i,j} + b(x) \text{ for all } i,j,
$
for some $a(x)>0$ and $b(x)\in\mathbb{R}$ that are independent of $i$ and $j$. Then the resulting \ours advantages are invariant:
$
    A'_{i,j} = A_{i,j} \text{ for all } i,j.
$
Consequently, uniform additive shifts and positive rescalings of token-level rewards across the sampled group do not create a learning signal in \ours. By contrast, outcome-level objectives built from absolute PRM scores are generally not invariant to the same transformations.
\end{proposition}

A detailed proof is provided in Appendix Section~\ref{appx:proof}. This proposition formalizes the core robustness intuition behind \ours: by operating on centered token-level comparisons derived from future-cumulated rewards rather than absolute PRM totals, the method removes one common avenue for PRM reward hacking.

\vspace{-2mm}
\subsection{The \ours Objective}
\vspace{-2mm}
Putting the pieces together, let $A_i^{\mathrm{GRPO}}$ denote the standard verifier-based GRPO advantage for trajectory $i$. The effective advantage applied to tokens in step $j$ of trajectory $i$ is
\[
\widetilde{A}_{i,j}=
\begin{cases}
A_i^{\mathrm{GRPO}}, & \text{if the sampled verifier rewards are mixed},\\
\dfrac{c_{i,j}-\bar c}{\sigma(c)}, & \text{if } r_1 = \cdots = r_G = 0.
\end{cases}
\]
where $c_{i,j} = \sum_{k=j}^{S_i} r_{i,k}$ and $\bar c = \frac{\sum_{i=1}^{G}\sum_{j=1}^{S_i} c_{i,j}}{\sum_{i=1}^{G}S_i}$.
Thus \ours does not replace RLVR. It preserves verifier-based learning whenever the verifier provides a preference signal, and only falls back to PRM guidance when verifier rewards are temporarily silent.

%% file: sections/Experiments.tex
\vspace{-2mm}
\section{Experiments and Results}
\vspace{-2mm}

We evaluate \ours along four axes: benchmark accuracy, mitigation of zero-gradient failures, robustness to reward hacking, and improvement in reasoning quality.
\vspace{-2mm}
\subsection{Setup}

\textbf{Models.}
We use Qwen2.5-1.5B-Instruct and Qwen2.5-7B-Instruct \citep{yang2024qwen25mathtechnicalreportmathematical} as policy models. For process supervision, we employ Qwen2.5-Math-PRM-7B \citep{prmlessons} and Math-Shepherd \citep{wang2024math}. Verifier-based outcome rewards are computed using the \textit{math-verify} library.

\textbf{Datasets.}
Training is performed on GSM8K \citep{cobbe2021training} and MATH \cite{hendrycks2021measuring}. Evaluation is conducted on six reasoning benchmarks: GSM8K \citep{cobbe2021training}, MATH-500 \cite{hendrycks2021measuring}, AMC, AIME-24 \cite{aime24}, MinervaMath \cite{lewkowycz2022solving}, and OlympiadBench \cite{he2024olympiadbench}. Unless otherwise specified, results are reported with zero-shot sampling-based decoding. Full training and evaluation settings, including hyperparameters for the 1.5B, 7B, PRM-enabled, and decoding configurations, are provided in Appendix Section~\ref{appx:training_details}.

\textbf{Baselines.}
We compare \ours against \begin{itemize*}[label={}, itemjoin={{; }}, itemjoin*={{; and }}]
\item \textsc{Vanilla-GRPO} \cite{shao2024deepseekmath}, which uses only verifier-based outcome rewards
\item \textsc{Dr-GRPO} \cite{liu2025understanding}, which fixes the bias in vanilla GRPO
\item \textsc{DAPO} \cite{yu2025dapo}, which introduces techniques such as decoupled clipping ranges and dynamic sampling
\item \textsc{PRM-as-ORM} \cite{zhang2025interplay}, which aggregates step-level PRM scores into a trajectory-level reward using product aggregation
\item \textsc{\ours (Simplified)}, which removes future cumulation in the step-reward advantage calculation
\end{itemize*}. We also compare directly with TreeRPO \cite{yang2025treerpo}, a state-of-the-art method for adding process supervision to GRPO, using the same training and evaluation setup.

\subsection{Performance on Reasoning Benchmarks}
\vspace{-2mm}

Table~\ref{tab:performance_results} summarizes the main benchmark results across six reasoning datasets. For both 1.5B and 7B, we observe a clear improvement in the performance of \ours compared to all baselines, including \textsc{PRM-as-ORM} and other GRPO variants. At the 7B scale, \ours achieves the best or tied-best performance, indicating that the \textbf{proposed objective scales favorably with model size.}

\textbf{Improvement from future-cumulated rewards:} Using the 1.5B model, we show that \ours outperforms \textsc{\ours (Simplified)}, which does not use future-cumulated rewards. \\
\textbf{Improvement from the verifier gate:} We also remove the gate from \ours, leaving only PRM rewards, and observe a significant drop in performance, demonstrating the value of retaining verifier-based supervision. To further show the usefulness of the gate, we run the same experiment for the \textsc{PRM-as-ORM} variant and observe the same trend. Overall, these results suggest that the gains from \ours come from a better form of credit assignment and a more robust optimization objective.

\textbf{Comparison with state-of-the-art intrinsic process-supervision approaches:} We provide a direct comparison with TreeRPO \cite{yang2025treerpo} in Appendix~\ref{appx:TreeRPO}. Like prior intrinsic process-supervision methods, TreeRPO relies on RLVR and intrinsic process signals and is therefore constrained by the limitations of RLVR and the policy model's capabilities, unlike \ours.

\textbf{Generalizability across PRMs:} Table~\ref{tab:math_shepherd_prm_results} evaluates performance under Math-Shepherd PRM, which is weaker than Qwen-PRM. The qualitative conclusion remains similar: performance does not collapse when the PRM changes, and \ours continues to outperform \textsc{PRM-as-ORM} and the GRPO baseline on most benchmarks. This result supports the broader claim that the improvements are not narrowly tied to Qwen-PRM but transfer to other PRMs as well.

\vspace{-2mm}
\begin{table}[htbp]
    \centering
    \caption{\textbf{Pass@1 accuracy on six mathematical reasoning benchmarks.} We report results for Qwen2.5-Instruct policies at 1.5B and 7B scale. The PRM model used is Qwen-PRM-7B. \ours is competitive with strong baselines at 1.5B and achieves the best or tied-best results across all reported 7B benchmarks, while clearly outperforming \textsc{PRM-as-ORM}. }
    \label{tab:performance_results}
    \small
    \setlength{\tabcolsep}{1pt}
    \resizebox{\linewidth}{!}{%
    \begin{tabular}{@{}lcccccc@{}}
        \toprule
         & \textbf{AIME} & \textbf{AMC} & \textbf{GSM8K} & \textbf{Math500} & \textbf{Minerva} & \textbf{Olympiad} \\
        \midrule
        \multicolumn{7}{c}{Qwen2.5-1.5B-Instruct} \\
        \midrule
        GRPO-Vanilla   \cite{shao2024deepseekmath}               & 0.00 & 15.10 & 71.11 & 39.80 & 7.35 & 15.29 \\
        Dr.GRPO   \cite{liu2025understanding}                     & 3.30 & 15.66 & 71.80 & 44.40 & 8.82 & 15.50 \\
        DAPO  \cite{yu2025dapo}                        & 0.00 & 14.10 & 71.49 & 40.20 & 7.72 & 14.07 \\
        PRM-as-ORM \cite{zhang2025interplay}         & 0.00 & 3.61  & 31.16 & 14.20 & 6.25 & 5.04  \\
        PRM-as-ORM (+Gate)         & 0.00 & 25.30  & 71.00 & 39.60 & 8.80 &  15.50 \\
        \ours (Simplified)  & 0.00 & 21.69 & 70.89 & 38.40 & 7.35 & 13.63 \\
        \ours (-Gate) & 3.33 & 7.23 & 67.7 & 34.00 & 7.35 & 10.37 \\
        \ours (Ours)  & 0.00 (+0.00\%) & \textbf{27.71 (+83.51\%)} & \textbf{72.10 (+1.39\%)} & 41.60 (+4.52\%) & \textbf{9.93 (+35.10\%)} & \textbf{15.50 (+1.37\%)} \\
        \midrule
        \multicolumn{7}{c}{Qwen2.5-7B-Instruct} \\
        \midrule
        GRPO-Vanilla \cite{shao2024deepseekmath}   & 6.67 & 42.17 & 92.33 & 61.67 & 18.38 & 30.67 \\
        PRM-as-ORM  \cite{zhang2025interplay}                   & 3.33 & 40.96 & 92.00 & 61.00 & 17.65 & 30.67 \\
        \ours (Ours)                 & \textbf{10.00 (+49.93\%)} & \textbf{45.78 (+8.56\%)} & 92.33 (+0.00\%) & \textbf{62.33 (+1.07\%)} & \textbf{20.22 (+10.01\%)} & \textbf{31.00 (+1.08\%)} \\
        \bottomrule
    \end{tabular}
    }
\end{table}

\vspace{-4mm}

\begin{table}[htbp]
    \centering
    \small
    \setlength{\tabcolsep}{4pt}
    \caption{\textbf{Pass@1 accuracy on six mathematical reasoning benchmarks.} We report results for Qwen2.5-Instruct policies at 1.5B and 7B scale. The PRM model used is Math-Shepherd PRM. }
    \begin{tabular}{@{}lcccccc@{}}
        \toprule
        Method & AIME & AMC & GSM8K & Math500 & Minerva & Olympiad \\
        \midrule
        PRM-as-ORM & 0.00 & 18.07 & 64.52 & 36.20 & 9.19 & 14.07 \\
        DeepSeek's PRM & \textbf{3.33} & \textbf{21.69} & \textbf{70.36} & \textbf{41.00} & 8.09 & \textbf{15.41} \\
        \bottomrule
    \end{tabular}
    \label{tab:math_shepherd_prm_results}
\end{table}

\textbf{Ablations on alternative ways of combining PRM rewards with outcome rewards} to clarify why the proposed verifier-gated design is empirically better.
\begin{itemize}[leftmargin=*,topsep=0in]
\setlength\itemsep{0pt}
\setlength\parsep{0pt}
\setlength\parskip{0pt}
\setlength\leftskip{0pt}
    \item In Appendix Section~\ref{appx:hard_subset}, we train only on the hard-subset prompts using PRM rewards. These hard subsets are constructed by selecting prompts that still receive zero reward after one epoch of GRPO training.
    \item In Appendix Section~\ref{appx:weighted_combination}, we study a separate always-on reward-mixing ablation that adds a weighted verifier reward to the step-level PRM reward on every prompt, without gating.
\end{itemize}

Both alternatives are more ad hoc than \ours and lack its clean robustness story. The hard-subset variant depends on pre-collecting degenerate prompts, making it difficult to use in a continual-learning setting, while the always-on reward-mixing variant introduces an additional hyperparameter and blends verifier and process supervision everywhere rather than separating them through the gate. These trade-offs further motivate the verifier-gated formulation used in \ours.

\begin{figure*}[ht!]
    \centering
    \includegraphics[width=\textwidth, trim={0 15 10 50}, clip]{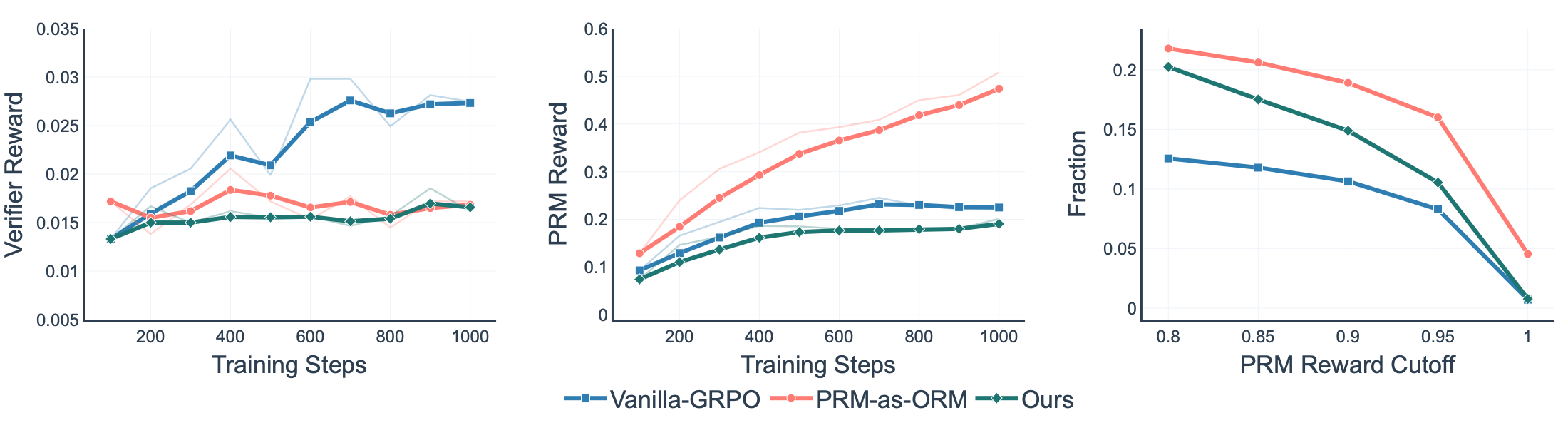}
    \caption{\textbf{(a): Comparison of verifier and PRM rewards over training steps to highlight reward hacking.} The left subplot shows the verifier reward, and the middle subplot displays the PRM reward across training steps. Each series is presented with Exponential Moving Average (EMA) Smoothed trend line and its corresponding raw data shown as a more transparent, fainter line in the background. \textbf{(b): Fraction of model responses above PRM reward cutoff.} Each line represents a method, showing the fraction of model responses that receive PRM reward above the cutoff value.}
    \label{fig:reward_over_time}
    \vspace{-1em}
\end{figure*}

\vspace{-2mm}
\subsection{Mitigating the Zero-Gradient Problem}
\vspace{-2mm}



\ours activates process supervision only when the verifier provides no learning signal and the standard GRPO update would otherwise collapse. As shown in Figure~\ref{fig:reward_over_time}(b), this yields a substantially faster reduction in the fraction of prompts that continue to receive zero verifier reward, indicating better exploration and a more reliable escape from the zero-gradient regime.

\vspace{-2mm}
\subsection{\ours is Resistant to Reward Hacking}
\vspace{-2mm}
A central concern with PRM-based supervision is reward hacking. To assess this, we evaluate model outputs using two metrics: (1) \textbf{verifier reward}, based on final-answer correctness, and (2) \textbf{PRM reward}, computed as the product of softmax-normalized step scores from the PRM.

Figure~\ref{fig:reward_over_time} tracks both rewards over training. For \textsc{PRM-as-ORM}, PRM reward rises rapidly while verifier reward remains nearly flat, indicating exploitation of PRM-specific biases rather than genuine task improvement. In contrast, \ours shows more closely aligned growth in verifier and PRM rewards, suggesting that higher PRM scores reflect real gains in correctness. This alignment is strong evidence that \ours is substantially less prone to reward hacking. This evaluation is conducted on the selected hard-prompt subset described in Appendix Section~\ref{appx:hard_subset}; accordingly, verifier reward remains close to zero and shows limited improvement for \textsc{Vanilla-GRPO}, \textsc{PRM-as-ORM}, and \ours.

Figure~\ref{fig:reward_over_time}(b) further reports the fraction of responses that receive high PRM reward despite incorrect final answers. Across a wide range of thresholds, \ours produces substantially fewer such responses than \textsc{PRM-as-ORM}, approaching the behavior of \textsc{Vanilla-GRPO} as the cutoff increases.

\textbf{Qualitative analysis of reward hacking.}
To complement the quantitative analysis above, we present qualitative responses to the same prompt generated by \textsc{PRM-as-ORM} (hacked response) and \ours (unhacked response) in Figure~\ref{fig:example_responses}.
Although the \textsc{PRM-as-ORM} response follows a structured, step-by-step format and receives a high process reward score, it contains a fundamental reasoning error that leads to an incorrect final answer. The response is verbose, but its reasoning remains unclear.
This behavior illustrates reward hacking: the policy optimizes for patterns favored by the process reward model rather than for faithful reasoning aligned with correctness. In contrast, \ours produces a coherent reasoning trajectory that correctly decomposes the problem and arrives at the correct solution.
These examples highlight how integrating process supervision at the objective level, rather than replacing outcome rewards, leads to more faithful reasoning behavior. Appendix Section~\ref{appx:reward_hacking_behavior} provides additional evidence, including lexical patterns and length distributions, showing that reward-hacked outputs often overproduce recurring stylistic tokens instead of improving solution quality.

\begin{figure*}[htp]
    \centering
    \begin{tcolorbox}[colback=gray!4,colframe=black!20,width=\textwidth,boxrule=0.4pt,arc=2pt]
        \textbf{Prompt.} A circle is inscribed in a square, then a square is inscribed in this circle, and finally, a circle is inscribed in this square. What is the ratio of the area of the smaller circle to the area of the larger square?
    \end{tcolorbox}
    \vspace{-0.2em}
    \begin{minipage}[t]{0.48\textwidth}
        \begin{tcolorbox}[title=Hacked Response,colback=red!3,colframe=red!35!black,boxrule=0.5pt,arc=2pt,top=3pt,bottom=3pt]
        \small
        Step 1: Identify the given shapes and their relationships. We start with a large square and an inscribed circle. Next, we have another square inscribed within this circle. Finally, there's another circle inscribed within this second square. We need to find the ratio of the area of the smaller circle to the area of the larger square. \\
        ...\\
        ...\\
        Step 5: .. \\
        Step 6: Find the radius of the third circle. The third circle is inscribed in the second square. The radius of the third circle is half the side length of the second square: $r_3 = \frac{k}{2} = \frac{s/\sqrt{2}}{2} = \frac{s}{2\sqrt{2}}$
        \end{tcolorbox}

    \end{minipage}\hfill
    \begin{minipage}[t]{0.48\textwidth}
        \begin{tcolorbox}[title=Unhacked Response,colback=green!3,colframe=green!35!black,boxrule=0.5pt,arc=2pt,top=3pt,bottom=3pt]
        \small
        Step 1: Consider a larger square with side length $s$.\\
        Step 2: The side length of this square will be equal to the diameter of the initial circle inscribed within it.\\
        ...\\
        Step 8: Radius of second circle $r_2 = d/2 = ((\sqrt(2*s))/2 )/2 = s/(2*\sqrt(2))$ \\
        Step 9: Area of the larger square $A_{\text{large-sq}} = s^2$. \\
        Step 10: Area of the smaller circle $ = \pi(r_2)^2 = \pi\left(s/2\sqrt{2}\right)^2 = \pi s^2/8$. \\
        Step 11: Ratio of the areas is $(\pi s^2/8) /s^2 = \pi/8$. This represents the ratio of the area of the smaller circle to that of the larger square.
        \end{tcolorbox}
    \end{minipage}

    \caption{\textbf{Qualitative comparison of hacked and unhacked responses for the same prompt.} The hacked response exhibits format-following behavior and locally plausible reasoning patterns that can inflate PRM scores without producing a correct solution. In contrast, the unhacked response follows a similar reasoning structure but reaches the correct geometric derivation and final answer.}
    \label{fig:example_responses}
    \vspace{-3mm}
\end{figure*}

\vspace{-2mm}
\subsection{\ours Training Improves the Reasoning Quality of the Model}
\vspace{-2mm}

\begin{wraptable}{r}{0.54\linewidth}
    \vspace{-1.1em}
    \centering
    \small
    \setlength{\tabcolsep}{4pt}
    \caption{Cross-PRM verification results using Qwen-PRM and Math-Shepherd PRM, with Math-Shepherd acting as an external evaluator.}
    \label{tab:cross_prm_verification}
    \begin{tabular}{lccc}
        \toprule
        Method & Qwen-PRM & Math-Shepherd & RLVR \\
        \midrule
        PRM-as-ORM & 0.5081 & 0.0331 & 0.017 \\
        \ours & 0.2011 & 0.1424 & 0.016 \\
        \bottomrule
    \end{tabular}
    \vspace{-0.6em}
\end{wraptable}

To test whether the gains from \ours reflect genuinely improved reasoning rather than overfitting to the training PRM reward, we additionally evaluate model outputs in a cross-PRM setting. Specifically, we score responses using both Qwen-PRM and Math-Shepherd, with Math-Shepherd serving as an external evaluator. Table~\ref{tab:cross_prm_verification} reports the average rewards from both PRMs together with the RLVR reward. Compared with \textsc{PRM-as-ORM}, \ours achieves a substantially lower score under Qwen-PRM but a markedly higher score under Math-Shepherd, while maintaining a comparable RLVR reward. This pattern suggests that \textsc{PRM-as-ORM} over-optimizes for the idiosyncrasies of Qwen-PRM, whereas \ours improves reasoning quality and logical consistency in a way that transfers better to an independent PRM. The high Qwen-PRM reward for \textsc{PRM-as-ORM} further indicates clear reward hacking with respect to the training PRM. A qualitative example is shown in Figure~\ref{fig:example_responses}.

%% file: sections/Related_Works.tex
\vspace{-2mm}
\section{Related Works}
\vspace{-2mm}
\textbf{Integrating process supervision into RLVR.}
Several recent works integrate process supervision into GRPO-style frameworks. A common strategy is to aggregate step-level PRM scores into a single trajectory-level reward, effectively replacing verifier rewards with PRM-derived outcome rewards \citep{lightman2023let, zhang2025interplay}. While this alleviates reward sparsity, it abandons the key advantage of verifier-based supervision and substantially increases susceptibility to reward hacking.
Other methods, such as TreeRPO \citep{yang2025treerpo}, apply GRPO at intermediate reasoning steps by explicitly constructing reasoning trees. Although this enables finer-grained credit assignment, the resulting tree depth grows exponentially with the number of reasoning steps, limiting practical scalability to very shallow reasoning processes.
Overall, prior work either replaces outcome rewards with process rewards or applies GRPO only on shallow step trees, trading off robustness or scalability. An extended discussion is in Appendix~\ref{appx:related_works}.

%% file: sections/Conclusion.tex
\vspace{-2mm}
\section{Conclusion and Limitations}
\vspace{-2mm}
In this work, we introduced \ours, a verifier-gated variant of GRPO that adds step-level supervision through future-cumulated rewards while preserving the reliability of verifier rewards. The method is simple: use standard verifier-driven GRPO when the verifier is informative, and use PRM-based step supervision only when all sampled trajectories receive zero verifier reward. This selective design strengthens training by reducing zero-gradient failures, limiting reward hacking relative to outcome-level PRM baselines, and improving cross-PRM transfer, suggesting better reasoning rather than overfitting to a single PRM.

\ours still depends on the availability and quality of process reward models. Although verifier-gating reduces reliance on imperfect PRMs, systematic PRM biases can still affect updates on degenerate prompts. The method also inherits the model's step segmentation, so poor segmentation can weaken credit assignment. Finally, our experiments focus on mathematical reasoning; broader evaluation is needed to test transfer to code generation, tool use, and long-horizon planning.


%% file: sections/Appendix.tex
\section{Limitations and Future Work}

\paragraph{Limitations.}
\ours still depends on the availability and quality of process reward models. Although verifier-gating limits over-reliance on imperfect PRMs, systematic PRM biases can still shape updates on degenerate prompts. Our formulation also treats reasoning steps as units defined by the model's generation format, so suboptimal step segmentation may blunt credit assignment. Finally, our experiments focus on mathematical reasoning; broader evaluation is needed to understand how well the method transfers to domains such as code generation, tool use, or long-horizon planning.

\paragraph{Future Work.}
A natural next step is to move beyond a binary gate and let the strength of process supervision depend on verifier confidence, uncertainty, or difficulty estimates. Another promising direction is to improve the policy and the PRM jointly, so that process supervision becomes more informative as training progresses. Extending verifier-gated process supervision to broader reasoning settings may also clarify when dense learned rewards complement verifiable objectives and when they become a liability.

\section{Training Details and Hyperparameters}
\label{appx:training_details}

Table~\ref{tab:training_hparams_main} summarizes the main training settings for the 1.5B and 7B experiments. Table~\ref{tab:training_hparams_prm} reports the shared PRM configuration used in PRM-enabled runs. Table~\ref{tab:eval_hparams} gives the default evaluation-time decoding settings.

\begin{table}[htbp]
    \centering
    \small
    \setlength{\tabcolsep}{4pt}
    \caption{\textbf{Training settings for the main 1.5B and 7B runs.}}
    \label{tab:training_hparams_main}
    \resizebox{\linewidth}{!}{%
    \begin{tabular}{@{}lll@{}}
        \toprule
        Setting & 1.5B training & 7B training \\
        \midrule
        Base model & Qwen2.5-1.5B-Instruct & Qwen2.5-7B-Instruct \\
        Reward & rule-based math verifier reward & rule-based math verifier reward \\
        Samples per prompt & 8 & 8 \\
        Training steps & 900 & 800 \\
        Prompt length & 256 & 1024 \\
        Generation length & 512 & 2048 \\
        Rollout batch size & 8 & 16 \\
        Micro rollout batch size & 4 & 4 \\
        Train batch size & 16 & 32 \\
        Micro train batch size & 2 & 4 \\
        Temperature / top-p & 1.0 / 1.0 & 1.0 / 1.0 \\
        Discount $\gamma$ & 1.0 & 1.0 \\
        Initial KL coef & 0.0 & 0.0 \\
        Learning rate & $1\times 10^{-6}$ & $1\times 10^{-6}$ \\
        LR scheduler & \texttt{cosine\_with\_min\_lr}  & \texttt{cosine\_with\_min\_lr} \\
        Precision & bf16 & bf16 \\
        LoRA & N/A & rank 64, alpha 128, dropout 0.0 \\
        LoRA target modules & N/A & \makecell[l]{\texttt{q\_proj}, \texttt{k\_proj}, \texttt{v\_proj}, \texttt{o\_proj},\\ \texttt{gate\_proj}, \texttt{up\_proj}, \texttt{down\_proj}} \\
        \bottomrule
    \end{tabular}%
    }
\end{table}

\begin{table}[htbp]
    \centering
    \caption{\textbf{Auxiliary training and evaluation settings.}}
    \label{tab:training_eval_aux}
    \begin{subtable}[t]{0.58\linewidth}
        \centering
        \small
        \setlength{\tabcolsep}{6pt}
        \caption{\textbf{PRM configuration used in PRM-enabled experiments.}}
        \label{tab:training_hparams_prm}
        \begin{tabular}{@{}ll@{}}
            \toprule
            PRM setting & Value \\
            \midrule
            PRM model & Qwen2.5-Math-PRM-7B \\
            PRM step separator & \texttt{<extra\_0>} \\
            Response step parsing & model responses split on \texttt{\textbackslash n\textbackslash n} \\
            PRM batch size & 8 \\
            PRM max steps & 16 \\
            \bottomrule
        \end{tabular}
    \end{subtable}\hfill
    \begin{subtable}[t]{0.36\linewidth}
        \centering
        \small
        \setlength{\tabcolsep}{8pt}
        \caption{\textbf{Default evaluation settings used for reported decoding results.}}
        \label{tab:eval_hparams}
        \begin{tabular}{@{}ll@{}}
            \toprule
            Setting & Value \\
            \midrule
            Temperature & 0.7 \\
            Top-$p$ & 0.95 \\
            Top-$k$ & 50 \\
            \bottomrule
        \end{tabular}
    \end{subtable}
\end{table}

\section{Preliminaries and Notations}
\label{sec:background}

\subsection{Proximal Policy Optimization (PPO)}
PPO \citep{schulman2017proximal} is an actor-critic RL method used for training an LLM policy. It stabilizes learning by constraining the update between the current policy $\pi_\theta$, the behavior policy $\pi_{\text{old}}$ used to generate sampled trajectories, and a fixed reference policy $\pi_{\text{ref}}$ used for KL regularization.
Given a prompt $x$, PPO samples a trajectory $y=(y_1,\dots,y_T) \sim \pi_{\text{old}}(\cdot \mid x)$ and optimizes a clipped surrogate objective. In language-model implementations, the likelihood ratio and clipping are applied at the token level rather than to the full-sequence probability. Let
\[
\rho_t(\theta)=\frac{\pi_\theta(y_t\mid x,y_{<t})}{\pi_{\text{old}}(y_t\mid x,y_{<t})},
\]
and let $A_t$ denote the token-level advantage estimate, typically derived from a learned value function $V_\phi$. The PPO objective is
\begin{equation}
\label{eq:ppo}
\begin{aligned}
\mathcal{L}_{\text{PPO}}(\theta)
=\;& \mathbb{E}\!\left[
\sum_{t=1}^{T}
\min\!\Big(
\rho_t(\theta)\,A_t,\;
\mathrm{clip}\!\big(\rho_t(\theta), 1-\epsilon, 1+\epsilon\big)\,A_t
\Big)
\right] \\
&- \beta\,\mathbb{E}\!\left[
\sum_{t=1}^{T} D_{\mathrm{KL}}\!\big[
\pi_{\theta}(\cdot \mid x,y_{<t})
\,\|\,
\pi_{\mathrm{ref}}(\cdot \mid x,y_{<t})
\big]
\right].
\end{aligned}
\end{equation}

The advantage is typically computed from a learned value function, for example through $A_t \approx R_t - V_\phi(x,y_{<t})$, where $V_\phi$ is trained jointly with the policy. The KL penalty is applied against the fixed reference policy $\pi_{\text{ref}}$ to keep the trained policy from drifting too far from its initialization.

By relying on a learned critic, PPO provides low-variance advantage estimates and enables temporal credit assignment. However, this is memory- and compute-intensive \citep{shao2024deepseekmath}. The next section discusses GRPO, which reduces memory requirements by removing the critic model and using verifiable reward functions.

\subsection{Group Relative Policy Optimization (GRPO)} \label{sec:grpo}

GRPO \citep{shao2024deepseekmath} optimizes language-model
policies using \emph{group-relative} feedback computed over
multiple sampled trajectories, eliminating the need for
learned reward or value models. Let $\pi_\theta(y \mid x)$ denote a policy that generates a
reasoning trajectory $y=(y_1,\dots,y_T)$ conditioned on a prompt $x$.
For each prompt, GRPO samples a group of $G$ trajectories
$y_i \sim \pi_{\text{old}}(\cdot \mid x), \  i = 1,\ldots,G,$
where $\pi_{\text{old}}$ denotes the behavior policy used to collect the sampled responses, and assigns each trajectory a verifiable outcome reward
$r_i = R_{\text{out}}(x,y_i), \  r_i \in \{0,1\},$
based solely on final-answer correctness.

Rather than using absolute rewards, GRPO computes a
group-relative advantage
$A_i = \frac{r_i - \bar{r}}{\sigma(r)},
\quad
\bar{r} = \frac{1}{G}\sum_{j=1}^G r_j,$
which centers rewards within the group and enforces relative
preferences among sampled trajectories. In LLM implementations, the policy ratio and clipping are applied token by token. For token $t$ in trajectory $i$, define
\[
\rho_{i,t}(\theta)=\frac{\pi_\theta(y_{i,t}\mid x,y_{i,<t})}{\pi_{\text{old}}(y_{i,t}\mid x,y_{i,<t})}.
\]
The GRPO objective for a single prompt $x$ is then
\begin{equation}
\begin{aligned}
\mathcal{L}_{\text{GRPO}}(\theta)
=
\mathbb{E}\Bigg[
&\frac{1}{G}
\sum_{i=1}^G \sum_{t=1}^{|y_i|}
\min\Big(
\rho_{i,t}(\theta) A_i,
\text{clip}\!\big(
\rho_{i,t}(\theta), 1-\epsilon, 1+\epsilon
\big) A_i
\Big)
\\
&\quad
- \beta\,
\frac{1}{G}\sum_{i=1}^G \sum_{t=1}^{|y_i|}
D_{\text{KL}}\!\left(
\pi_\theta(\cdot\mid x,y_{i,<t})
\,\|\, \pi_{\text{ref}}(\cdot\mid x,y_{i,<t})
\right)
\Bigg],
\end{aligned}
\label{eq:grpo}
\end{equation}
where $\pi_{\text{ref}}$ is a fixed reference policy used only for KL regularization.

Thus, GRPO uses outcome-level rewards to define a trajectory-level advantage, uses the behavior policy $\pi_{\text{old}}$ in the clipped importance ratio, and uses the fixed reference policy $\pi_{\text{ref}}$ in the KL penalty. By relying exclusively on outcome-level supervision,
GRPO assigns the same advantage to all steps within a
trajectory and provides no explicit mechanism for step-level
credit assignment. Moreover, when all sampled trajectories
receive identical rewards, the group-relative advantages
collapse to zero, yielding no learning signal. These
limitations motivate the use of dense process supervision,
which we analyze in Section~\ref{sec:prm_problems}.

\subsection{Outcome Rewards and Process Rewards}

Reasoning-focused reinforcement learning uses two common kinds of feedback that differ in granularity.

\textbf{Outcome rewards.}
Outcome rewards assign a single scalar to the full trajectory,
$R_{\text{out}}(x,y)$.
In RLVR, this reward typically depends only on whether the final answer is correct. Outcome rewards are attractive because they are reliable and difficult to hack when the verifier is trustworthy, but they are sparse and provide no direct information about which intermediate reasoning steps were useful.

\textbf{Process rewards.}
Process rewards instead score intermediate reasoning steps. A process reward model (PRM) provides feedback of the form
\[
    r_{\text{proc}}(x, y_{<t}, y_t),
\]
which evaluates the local quality of the next step given the prompt and the partial reasoning trace. In principle, process rewards offer denser learning signals and finer credit assignment than outcome rewards. However, PRMs are learned models rather than ground-truth verifiers, so they can be biased, miscalibrated, and susceptible to reward hacking.

The central challenge of this paper is therefore not whether process supervision is useful, but how to incorporate it into GRPO without discarding the reliability of verifier-based outcome rewards.

\section{Related Works}\label{appx:related_works}

\subsection{Outcome Supervision for Reasoning Models}
Outcome-based supervision has been widely adopted for training reasoning models, particularly through reinforcement learning with verifier rewards (RLVR), where a deterministic verifier assigns rewards based on final-answer correctness \citep{cobbe2021training, guo2025deepseek}. Group Relative Policy Optimization (GRPO) \citep{shao2024deepseekmath} builds on this paradigm by using group-relative normalization to eliminate the need for learned reward or value models, yielding strong robustness to reward hacking and favorable scalability.

However, outcome-only supervision is inherently sparse. GRPO assigns a single scalar reward to an entire reasoning trajectory and applies the same learning signal uniformly across all steps. As a result, GRPO provides no explicit mechanism for step-level credit assignment and can fail to make progress on difficult problems where most sampled trajectories receive identical rewards \citep{sun2025rl}. These limitations motivate the incorporation of denser supervision signals.

\vspace{-2mm}
\subsection{Process Supervision and Process Reward Models}
Process supervision addresses reward sparsity by providing feedback on intermediate reasoning steps rather than only evaluating final answers \citep{uesato2022solving, lightman2023let}. This supervision is typically implemented using Process Reward Models (PRMs), which score individual reasoning steps. Both discriminative PRMs \citep{lightman2023let, wang2024math, zhang2025lessons} and generative PRMs \citep{zheng2023judging, khalifa2025process, zhao2025genprm} have been shown to improve credit assignment and learning efficiency.

Despite their promise, PRMs are learned models trained on imperfect supervision and are therefore vulnerable to bias and distributional mismatch. Policies optimized directly using PRM outputs can exploit systematic artifacts in the reward model, leading to reward hacking behaviors that do not correspond to genuine improvements in reasoning quality \citep{anonymous2025cut, baker2025monitoring}. This makes naive reliance on PRMs as primary reward signals problematic.

\section{Invariance Proof}\label{appx:proof}

\begin{proposition}[Invariance to prompt-wise positive affine PRM transformations]
\label{app_prop:invariance_inflation}
Fix a prompt $x$ and let $\{c_{i,j}\}$ denote the future-cumulated rewards assigned to every token in step $j$ of trajectories $y_i \sim \pi_\theta(\cdot\mid x)$, with $j=1,\dots,S_i$.
\ours defines token-level advantages
\[
A_{i,j}=\frac{c_{i,j}-\bar c}{\sigma(c)},\qquad
\bar c=\frac{\sum_{i=1}^{G}\sum_{j=1}^{S_i} c_{i,j}}{\sum_{i=1}^{G} S_i},
\]
with $\sigma(c)$ the standard deviation over all future-cumulated rewards in the sampled group for $x$.
Suppose a policy update induces a \emph{prompt-wise positive affine transformation} of these rewards, i.e.,
\[
c'_{i,j} = a(x)c_{i,j} + b(x)\quad \text{for all } i,j,
\]
for some $a(x)>0$ and $b(x)\in\mathbb{R}$ independent of $i,j$.
Then the \ours advantages are invariant:
\[
A'_{i,j} = A_{i,j}\quad \text{for all } i,j.
\]
Consequently, uniform additive shifts and positive rescalings of these PRM-derived rewards do not create a learning signal in \ours. By contrast, outcome-level objectives built from absolute PRM scores are generally not invariant to the same transformations.
\end{proposition}

\begin{proof}[Proof sketch]
Under the assumed transformation, the baseline undergoes the same affine map:
\[
\bar c' = \frac{\sum_{i=1}^{G}\sum_{j=1}^{S_i} (a c_{i,j}+b)}{\sum_{i=1}^{G} S_i} = a\bar c + b.
\]
Because $a>0$, standard deviation scales linearly under this transformation, so $\sigma(c')=a\sigma(c)$.
Therefore,
\[
A'_{i,j}=\frac{c'_{i,j}-\bar c'}{\sigma(c')}
=\frac{(a c_{i,j}+b)-(a\bar c+b)}{a\sigma(c)}=\frac{c_{i,j}-\bar c}{\sigma(c)}=A_{i,j}.
\]
Thus, any update that only shifts or positively rescales PRM-derived rewards uniformly within a prompt produces exactly the same normalized advantages and therefore cannot amplify itself through training. In contrast, outcome-level objectives optimize an absolute scalar built from the raw PRM scores, such as a sum, product, or minimum, and these quantities generally change under the same affine transformation. This leaves such objectives more exposed to over-optimization of PRM-correlated but non-causal features.
\end{proof}

\section{Additional Results}

\subsection{Results on TreeRPO}\label{appx:TreeRPO}

Table~\ref{tab:treerpo_results} reports the direct comparison with TreeRPO under the same setup. \ours significantly improves over TreeRPO because it is not limited to intrinsic process signals tied to the policy model and just RLVR rewards.


\begin{table}[htbp]
    \centering
    \caption{Overall Pass@1 (Avg@16) performance of the Qwen2.5-Math series with Qwen2.5-Math-1.5B as the base model.}
    \label{tab:treerpo_results}
    \begin{tabular}{@{}lcccc@{}}
        \toprule
        Method & AIME24 & MATH500 & Minerva & Macro Accuracy \\
        \midrule
        TreeRPO & 16.8 & 70.7 & 24.0 & 37.2 \\
        \ours & 23.3 & 70.0 & 26.0 & 39.8 \\
        \bottomrule
    \end{tabular}
\end{table}


\subsection{Reward-Hacking Behavior Analysis}\label{appx:reward_hacking_behavior}

We further analyze the behavior of reward-hacked responses under \textsc{PRM-as-ORM}. We define \emph{hacking traces} as responses that are incorrect according to the \textit{math\_verify} verifier but still receive high PRM reward, and \emph{non-hacking traces} as incorrect responses that receive low PRM reward. This comparison helps characterize which response patterns the PRM tends to overvalue.

Figure~\ref{fig:reward_hacking_word_cloud} shows a word cloud of the top 80 words that are most likely to appear in hacking traces relative to non-hacking traces. The dominant signal is not better mathematical content, but recurring stylistic and structural words that appear disproportionately often in high-PRM incorrect responses.

Figure~\ref{fig:reward_hacking_length} compares hacking and non-hacking traces by total number of words and total number of reasoning steps. In general, hacking traces are longer on both axes, suggesting that the PRM can overvalue verbose, highly structured responses even when they do not produce correct final answers.

\begin{figure}[htbp]
    \centering
    \includegraphics[width=0.72\linewidth]{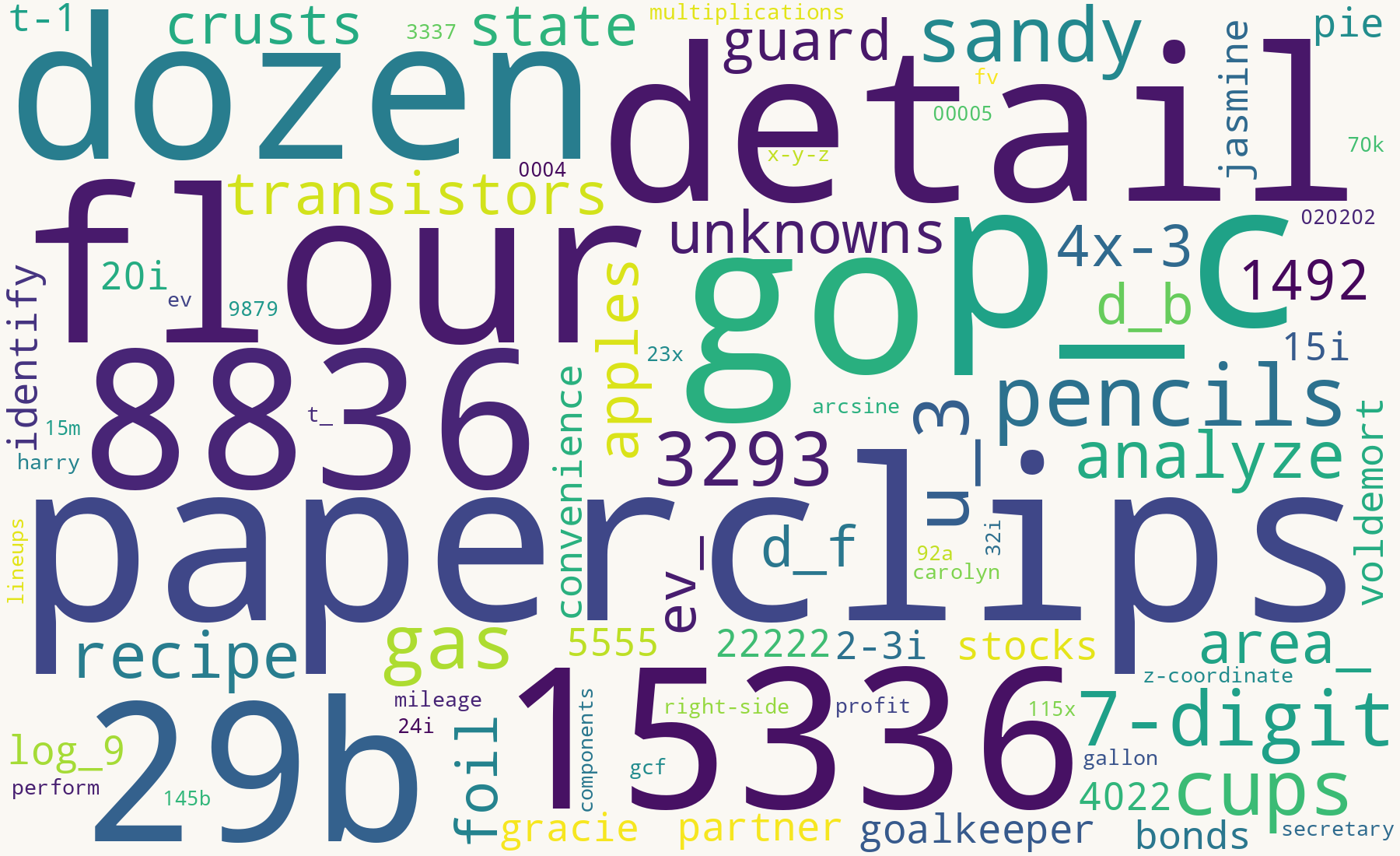}
    \caption{\textbf{Word cloud of the top 80 words most associated with reward-hacked traces.} Hacking traces are incorrect responses that receive high PRM reward; non-hacking traces are incorrect responses that receive low PRM reward.}
    \label{fig:reward_hacking_word_cloud}
\end{figure}

\begin{figure}[htbp]
    \centering
    \includegraphics[width=0.82\linewidth]{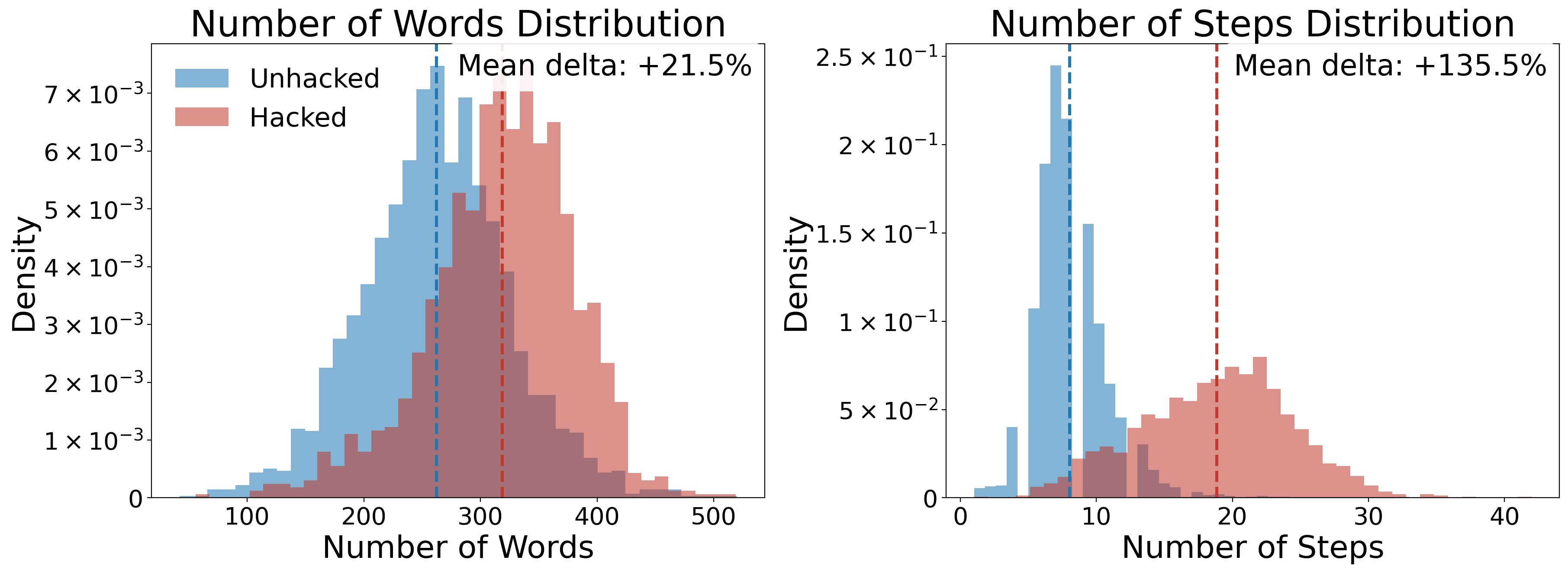}
    \caption{\textbf{Distribution of trace lengths for hacked and unhacked responses.} Reward-hacked traces are generally longer, both in number of words and in number of reasoning steps.}
    \label{fig:reward_hacking_length}
\end{figure}



\subsection{Training Only on Hard-Subset of Data}\label{appx:hard_subset}

\textbf{Hard Prompt Selection.}
To study zero-gradient failures, we construct hard subsets of GSM8K and MATH by first training Qwen2.5-1.5B-Instruct using standard GRPO for one epoch and then selecting prompts for which all sampled trajectories receive zero verifier reward. This results in 280 hard GSM8K prompts and 2,969 hard MATH prompts for the model. 

Table~\ref{tab:main_results} reports Pass@1 accuracy across all benchmarks. \ours consistently outperforms \textsc{Vanilla-GRPO} on both in-domain and out-of-distribution tasks, with particularly strong gains on hard subsets where outcome-only supervision struggles.
Compared to \textsc{PRM-as-ORM}, \ours achieves comparable or better performance while avoiding the degradation observed on certain benchmarks. These results indicate that step-level credit assignment improves learning efficiency, while preserving the robustness of verifier-based supervision.
Notably, \ours trained on MATH generalizes better to AMC, MinervaMath, and OlympiadBench than both baselines, suggesting improved reasoning transfer beyond the training distribution.

\begin{table*}[ht!]
\centering
\caption{\textbf{Pass@1 accuracy on six mathematical reasoning benchmarks.} \ours consistently outperforms \textsc{Vanilla-GRPO} and matches or surpasses \textsc{PRM-as-ORM} on most benchmarks.}
\label{tab:main_results}
\scriptsize
\setlength{\tabcolsep}{4pt}
\resizebox{\textwidth}{!}{%
\begin{tabular}{lcccccc}
\toprule
\multicolumn{7}{c}{\textbf{Policy: Qwen2.5-1.5B-Instruct | Training Data: MATH Hard Subset}} \\
\midrule
Method & AIME-24 & AMC & GSM8K & MATH500 & Minerva & Olympiad \\
\hline
\textsc{Vanilla-GRPO} 
& 0.00 
& 7.23 
& 49.43 
& 20.00 
& 5.51 
& 7.26 \\

\textsc{PRM-as-ORM} 
& 0.00 
& 18.07 (+10.84 \%) 
& 64.52 (+15.09 \%) 
& \textbf{36.40 (+16.40 \%)} 
& 5.51 (+0.00 \%) 
& \textbf{15.70 (+8.44 \%)} \\

\ours 
& 0.00 
& \textbf{22.89 (+15.66 \%)}
& \textbf{68.01 (+18.58 \%)} 
& 35.00 (+15.00 \%) 
& \textbf{8.09 (+2.58 \%)} 
& 14.37 (+7.11 \%) \\
\midrule
\multicolumn{7}{c}{\textbf{Policy: Qwen2.5-1.5B-Instruct | Training Data: GSM8K Hard Subset}} \\
\hline
\textsc{Vanilla-GRPO} 
& 0.00 
& \textbf{15.66} 
& 61.49 
& 34.00 
& 6.62 
& 13.19 \\

\textsc{PRM-as-ORM} 
& 0.00
& 13.25 (-2.41 \%) 
& \textbf{64.52 (+3.03 \%)} 
& 34.20 (+0.20 \%) 
& 7.35 (+0.73 \%) 
& 13.63 (+0.44 \%) \\

\ours 
& \textbf{3.30 (+3.30 \%)} 
& 15.66 (+0.00 \%) 
& 64.37 (+2.88 \%) 
& \textbf{36.00 (+2.00 \%)} 
& \textbf{8.82 (+2.20 \%)} 
& \textbf{15.41 (+2.22 \%)} \\
\bottomrule
\end{tabular}}
\end{table*}

\subsection{Training using weighted combination of Verifier and Process Rewards}\label{appx:weighted_combination}

As a separate ablation from the verifier-gated design of \ours, we consider an always-on reward-mixing variant that combines verifier-based outcome rewards and step-level process supervision on every prompt, without gating. Concretely, we introduce a scalar weight $\lambda \in [0,1]$ and augment the step-level process reward by adding $\lambda R_{\text{out}}(x,y)$ at every step, where $R_{\text{out}}(x,y) \in \{0,1\}$ denotes the verifier reward for the full trajectory. Larger values of $\lambda$ therefore increase the influence of outcome correctness on every step, while $\lambda=0$ corresponds to pure process-based step supervision.

We examine whether this explicit always-on mixing of verifier and process rewards improves performance. We vary the weight $\lambda$ and evaluate the resulting model across benchmarks. As shown in Figure~\ref{fig:lambda_ablation}, increasing $\lambda$ does not yield consistent improvements and, in some cases, degrades performance, particularly on GSM8K and AMC. This suggests that naively injecting the same outcome reward into every step is not as effective as the verifier-gated design of \ours. Overall, this ablation supports our main design choice: it is better to keep verifier supervision and process supervision separated by the gate than to combine them everywhere through a fixed mixing weight.

\begin{figure}[htbp]
    \centering
    \includegraphics[width=0.6\linewidth]{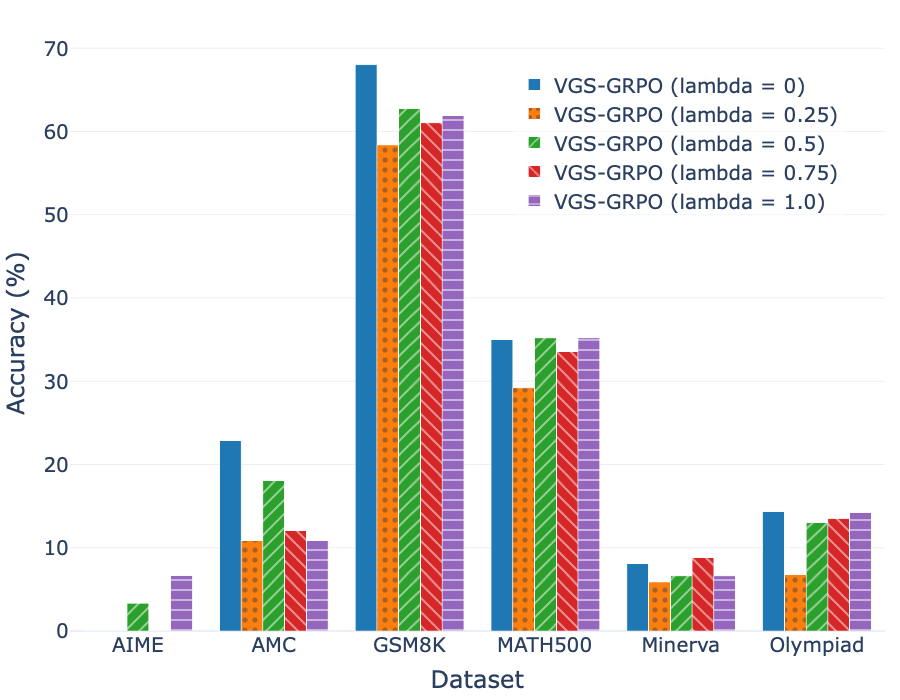}
    \caption{\textbf{Accuracy across datasets for different outcome-reward weights $\lambda$ in an always-on reward-mixing ablation.} This grouped bar chart shows the accuracy obtained when a verifier reward term $\lambda R_{\text{out}}(x,y)$ is added to the step-level process reward at every step, without gating, for $\lambda \in \{0, 0.25, 0.5, 0.75, 1.0\}$. Accuracy values are reported as percentages.}
    \label{fig:lambda_ablation}
    \vspace{-2mm}
\end{figure}







